\documentclass[letterpaper, 10 pt, conference]{ieeeconf}  
\IEEEoverridecommandlockouts                              
\usepackage{graphics,epsfig}
\usepackage{amsmath,amssymb,bm,bbm}
\usepackage{arydshln}

\usepackage{blindtext}

\usepackage{amsthm}

\usepackage{xcolor}
\usepackage[hidelinks]{hyperref}

\newtheorem{problem}{Problem}
\newtheorem{definition}{Definition}
\newtheorem{assumption}{Assumption}
\newtheorem{theorem}{Theorem}
\newtheorem{lemma}{Lemma}
\newtheorem{remark}{Remark}

\usepackage[ruled,vlined]{algorithm2e}
\SetKwInOut{Input}{input}
\SetKwInOut{Output}{output}
\SetKwRepeat{Do}{do}{while}

\SetCommentSty{mycommfont}

\usepackage{booktabs}

\usepackage{tikz}

\title{\LARGE \bf
A Global Games-Inspired Approach to 
Multi-Robot Task Allocation for Heterogeneous Teams
}

\author{Logan E. Beaver, \textit{Member, IEEE}
\thanks{*This work was supported by the Virginia Space Grant Consortium.}
\thanks{Logan E. Beaver is with Department of Mechanical \& Aerospace Engineering,
        Old Dominion University, Norfolk, VA 23529 USA
        {\tt\small lbeaver@odu.edu}}
}

\usepackage{amsthm}

\begin{document}

\maketitle
\thispagestyle{empty}
\pagestyle{empty}

\begin{abstract}
In this article we propose a game-theoretic approach to the multi-robot task allocation problem using the framework of global games.
Each task is associated with a global signal, a real-valued number that captures the task execution progress and/or urgency. 
We propose a linear objective function for each robot in the system, which, for each task, increases with global signal and decreases with the number assigned robots.
We provide conditions on the objective function hyperparameters to induce a mixed Nash equilibrium, i.e., solutions where all robots are not assigned to a single task.
The resulting algorithm only requires the inversion of a matrix to determine a probability distribution over the robot assignments.
We demonstrate the performance of our algorithm in simulation and provide direction for applications and future work.
\end{abstract}


\section{Introduction}

As we continue deploying robots in real outdoor environments, there is a growing interest in automating data collection and operations in inaccessible environments, such as remote sensing \cite{naderi2022sharing}, environmental data collection \cite{dunbabin2012robots}, ecological monitoring \cite{notomista2022multi}, and extraterrestrial environments \cite{beaver2025MRTA}.
Automated solutions to these problems require a robust, resilient, and flexible approach to allocate many robots to a variety of tasks over long time durations.

These problems fall broadly in the domain of Multi-Robot Task Allocation (MRTA).
MRTA is one of the most challenging open problems in mutli-robot systems research \cite{khamis2015multi}; despite the great number of MRTA algorithms, critical components, such as dynamic task allocation and the use of heterogeneous robots, present a significant open problem in the literature.
In this work, we consider a general allocation approach that allows one or many heterogeneous robots to be assigned to multiple tasks with time-varying dynamics.
Specifically, we present a global games-inspired strategy for multi-robot task allocation that is scalable, computationally efficient, and performant.
Our approach is also resilient to the spontaneous addition and removal of robots to the system, as well as the dynamic addition and removal of available tasks.
Furthermore, our approach allows heterogeneity in the capability of the robots, with more capable robots being assigned to tasks first.

In general, MRTA requires solving a combinatorial optimization problem to assign robots to tasks, and this does not lend itself to the kind of reactive, robust, and resilient formulation required for dynamic allocation.
This is apparent from a recent MRTA survey article \cite{chakraa2023optimization}, which estimates that at least $75\%$ of the MRTA literature requires either 1) multiple rounds of coordination between robots or 2) offline computation of a static strategy.
These both require either a centralized authority to solve a combinatorial optimization problem, or it necessitates multiple rounds of communication until all robots reach a consensus on the task assignment.
As a consequence, any major change in the system (e.g., the addition or removal of robots) requires the assignment to be re-computed.
In contrast, our approach only requires robots to estimate the intensity of a scalar \textit{signal} related to each task and the number of robots currently assigned to it.
This can be achieved by a centralized database that stores information but perform any computations on behalf of the robots, or through a standard consensus mechanism.

While there have been advances in robot allocation for dynamic systems \cite{notomista2019optimal}, game-theoretic approaches make up only about $4\%$ of published solutions \cite{chakraa2023optimization}.
A majority of approaches consider only the formation of coalitions for task assignment as a game \cite{MARTIN2023104314}, which clusters robots into teams that are then assigned to tasks using combinatorial optimization.
In fact, \cite{chakraa2023optimization} contains only a singular reference to global games for task allocation under noisy observations \cite{kanakia2016modeling}, along with our previous work \cite{beaver2025MRTA}, is the only MRTA framework to explicitly use global games.

In a global game, robots in the system measure a global \textit{signal} (or stimulus) that corresponds to a task, e.g., the amount of energy remaining for a foraging task \cite{Krieger2000} the intensity of a fire in a fire-fighting task \cite{kanakia2016modeling}.
It has been proven that, under reasonable conditions, the Nash equilibrium for a global game reduces to a threshold strategy \cite{kanakia2016modeling,Mahdavifar2018GlobalSharing}.
This means that robots simply compare their measurement of the global signal to an internal threshold, then assign themselves to the task when the threshold is exceeded.
A similar kind of approach was used in a biologically-inspired colony maintenance problem \cite{krieger2000call} without the formalism of global games for assignment.
While a threshold strategy based on a global signal is appealing, our previous work demonstrated that this is neither resilient nor adaptive \cite{beaver2025MRTA}.
Global games have desirable properties for a one-shot assignment, for example, modeling a bank run or debt crisis \cite{Mahdavifar2018GlobalSharing}.
However, these strategies do not adapt to changes in the global signal (or stimulus), and they are not robust to the addition or removal of robots \cite{beaver2025MRTA}.

To summarize, the contributions of this article are as follows:
\begin{enumerate}
    \item a global-games based framework for the assignment of multiple heterogeneous robots to multiple tasks without explicit communication (Algorithm \ref{alg:elimination}),
    \item a reduction of the heterogeneous assignment problem to the homogeneous case (Lemma \ref{lma:reduction}),
    \item proof that more capable robots are always assigned to tasks first (Remark \ref{rmk:capable}), and
    \item illustrative numerical examples that demonstrate the dynamic and robust allocation of our algorithm (Section \ref{sec:sim}).
\end{enumerate}

The remainder of the article is organized as follows.
We present our global game framework and working assumptions in Section \ref{sec:problem}.
We derive the  mixed Nash Equilibrium for assignment and propose a numerically tractable algorithm to assign robots to tasks in real time in Section \ref{sec:solution}.
We present two simulation result sin Section \ref{sec:sim} that demonstrates the resilience of our approach to the removal of robots and shows the performance of our approach for a heterogeneous system.
Finally, we present concluding remarks and directions for future work in Section \ref{sec:conclusion}.

\subsection{Notation}

In this work we consider repeated rounds of game with a mixed Nash equilibrium.
For simplicity, we use the following notational rules for variables,
\begin{itemize}
    \item Lower case letters denote a variable known a priori.
    \item Upper case letters denote random variables.
    \item Subscripts denote task index(es).
    \item Superscripts denote the index of a homogeneous group of robots.
\end{itemize}

Under these rules, $n_k$ is the number of robots currently assigned to task $k$, while $N_k$ is a random variable that describes the number of robots that will be assigned to task $k$ after the assignment, and $p_k^i$ is the probability of assigning robot $i$ to task $k$.

We use the absolute value sign to sum vector elements in the heterogeneous case, so $\big|\bm{n}_k\big| = \big|[n_k^1,\, n_k^2,\, \dots]^{\intercal}\big| := n_k^1 + n_k^2 + \dots$ is the total number of robots from each group $1, 2, \dots,$ that are assigned to task $k$.

\section{Problem Formulation} \label{sec:problem}

We consider a team of $N$ mobile robots seeking to complete up to $M$ tasks in the local environment.
Each of the $m \in\{1, 2, \dots, M\}$ tasks has an associated scalar signal,
\begin{equation} \label{eq:signal}
    s_m(t) \in \left[0, 1\right],
\end{equation}
which is a scalar metric for the ``completeness'' of the task.
Specifically, $s_m \approx 1$ implies that task $m$ is satisfied and $s_m \approx 0$ implies that the task is at risk of failure.
Each signal is data-driven with possibly unknown dynamics, and its exact dynamics are application dependent.
For example, $s_k$ may be the (normalized) voltage in a battery that changes with weather and and applied load; or in a surveillance problem, $s_k$ could be inversely proportional to the time that a location was last visited by a robot.

We propose a \textit{mechanism design} solution to allocate robots to tasks within a global games framework, i.e., we design the utility functions for each robot to ensure the Nash equilibrium of the induced game allocates more capable robots to tasks, and more robots are assigned to tasks with signals closer to zero.
We also seek to avoid trivial solutions where all robots are assigned to a single task, and allow for some robots to be assigned to no task; this makes our approach robust to the addition and removal of robots and tasks.
To start, introduce a definition for our game \cite{chremos2020game}.
\begin{definition} \label{def:game}
    A finite normal-form game is a tuple 
    \begin{equation*}
        \mathcal{G} = \big(\mathcal{I}, \mathcal{S}, u_i \big),
    \end{equation*}
    where
    \begin{itemize}
        \item $\mathcal{I} = \{1, 2, \dots, N\}$ is the set of players,
        \item $\mathcal{S}=\{0,1,\dots,M\}$ is the set of strategies, and
        \item $u_i,\, i\in\mathcal{I}$ is the utility function for each player.
    \end{itemize}
\end{definition}
Note that the set of strategies in Definition \ref{def:game} is a collection of integer variables, where $0$ corresponds to doing no task (i.e., \textit{idling}) and  $1, 2, \dots, M$ corresponds to an assignment to that task.
Next, we define the utility for each robot and it's main properties.
\begin{definition} \label{def:utility}
    Each robot $i\in\mathcal{I}$ takes an action $a\in\mathcal{S}$ that yields a utility,
    \begin{equation}
        u_i(a, \bm{n}, \bm{s}) \to \mathbb{R},
    \end{equation}
    where $\bm{n} = [n_1,\,n_2,\, \dots,\, n_M]$ counts the number of robots assigned to each task and $\bm{s} = \big[s_1(t),\, s_2(t),\, \dots,\, s_M(t)\big]$ is a vector of global signals.
\end{definition}

Within this framework, we employ the following technical assumptions for our approach.
\begin{assumption} \label{smp:monotonicity}
    The utility function for each robot (Definition \ref{def:utility} is
     (i)  strictly decreasing in each element of $\bm{s}$, and
    (ii) strictly decreasing in each element of $\bm{n}$.
\end{assumption}
\begin{assumption} \label{smp:information}
    Each robot has access to a database of information containing the signal vector $\bm{s} = [s_1, \dots, s_N]$ and the assignment vector $\bm{n} = [n_1, \dots, n_N]$.
\end{assumption}

\begin{assumption} \label{smp:hysteresis}
    Only idle robots (assigned to task $0$) participate in the assignment game, and each robot remains assigned to their task for some non-zero interval of time before switching back to the idle task.
\end{assumption}

Assumptions \ref{smp:monotonicity}--\ref{smp:information} are technical assumptions required for our problem to be well-posed and for the robots to be able to generate a solution.
We demonstrate how violating Assumption \ref{smp:monotonicity} can lead to cascading system failures in our previous work \cite{beaver2025MRTA}.
Assumption \ref{smp:information} is required for the robots to compute a solution, but it could be relaxed, for example, by implementing a consensus strategy and considering the communication network between robots.
Assumption \ref{smp:hysteresis} is the most restrictive; it ensures that robots do not rapidly switch between their assigned task and idling while also simplifying the Nash equilibrium calculations.
Other aspects of our framework, e.g., the domain of our signal $s_k(t) \in [0, 1]$ or the specific utility function \eqref{eq:utility}, can be relaxed by following our analysis.

Here we consider a linear utility functions for each robot.
This enables us to generate closed-form analytical solutions for the assignment.
Nonlinearities can be introduced to the assignment through the design of the signal function $s_k(t)$, which maps some state variables of the system to a scalar value.
In particular, we consider for each robot $i\in\mathcal{I}$,
\begin{equation} \label{eq:utility}
    u_i(a, \bm{n}, \bm{s}) = \sum_{k=1}^M \mathbbm{1}_k\Big(\frac{\gamma_k - n_k}{\gamma_k} - s_k - c_k^i \Big),
\end{equation}
where for each task $k$,
$n_k$ is the number of robots assigned to the task, $\gamma_k$ is a tuning parameter, $c^i_{k}$ is the cost of assigning robot $i$ to task $k$, and $\mathbbm{1}_k \in\{0, 1\}$ is an indicator function that is $1$ when the robot is assigned to task $k$ and $0$ otherwise.

The intuition behind the utility function \eqref{eq:utility} is as follows.
The first term normalizes $n_k$ by some $\gamma_k$ that, all else being equal, roughly corresponds to the fraction of all robots that should be assigned to task $k$.
The cost $c_k^i$ is the cost of assigning robot $i$ to task $k$, which can capture both the capability of robot $i$ to do the task, the energy required to do the task, and indirect factors such as wear-and-tear and failure risk.
The the signal $s_k(t)$ satisfies Assumption \ref{smp:monotonicity}, and $\mathbbm{1}_k$ is an indicator that ensures robots consider the cost of each task individually.

Finally, we consider the definitions of strictly dominant and mixed strategy equilibrium, which will characterize the Nash equilibrium induced by our utility function \eqref{eq:utility}.

\begin{definition}[Strictly Dominant Strategy] \label{def:dom-strat}
The action $a_i^*$ is a \textit{strictly dominant strategy} for a given $\bm{s}$ if and only if
\begin{equation}
    u_i(a_i^*, n, \bm{s}) > u_i(a_i', n, \bm{s}),
\end{equation}
for  all $n$ and $a_i' \neq a_i^*$. 
\end{definition}

\begin{definition}
[Mixed Strategy] \label{def:mix-strat}
A mixed strategy for a given $\bm{s}$ is a probability distribution over a support $\delta\subseteq\mathcal{S}$ such that,
\begin{equation}
    \mathbb{E}\Big[u_i(a, n, \bm{s}) \Big] = \mathbb{E}\Big[u_i(a', n, \bm{s}) \Big],
\end{equation}
for all $a, a' \in\delta$.
\end{definition}

As described in the problem formulation, we wish to avoid strategies where a strictly dominant strategy exists.
This is a trivial assignment problem where all robots are assigned to a single task.
Instead, we develop a useful mixed strategy to assign robots to tasks based on the current value of the signal vector $\bm{s}(t)$ and the number of robots currently assigned to each task.

\section{Solution Approach} \label{sec:solution}

To derive our solution to the assignment problem, we will first derive the Nash equilibrium induced by our linear objective function (Definition \ref{def:utility}) and prove that it always exists.
To achieve this, we start with the marginal utility of robot $i\in\mathcal{A}$ selecting between tasks $0,k,j\in\mathcal{S}$,
\begin{align}
    \pi^i_{k0} &= \frac{\gamma_k - N_k}{\gamma_k} - s_k - c_k^i, \label{eq:pik0} \\
    \pi^i_{kj} &= \frac{\gamma_k - N_k}{\gamma_k} - s_k - c_k^i - \frac{\gamma_j - N_j}{\gamma_j} + s_j + c_j^i.\label{eq:pikj}
\end{align}
Here, $N_k, N_j$ are \textit{random variables} corresponding to the number of robots that will be assigned to tasks $k$ and $j$, respectively.
To find the mixed strategy Nash equilibrium, we first find the solution for a homogeneous system following our existing approach \cite{beaver2025MRTA}.
Then, we find the general solution for a heterogeneous team and prove that it reduces to the homogeneous solution with minor modifications.

\subsection{Homogeneous Systems}
To assign a team of homogeneous robots to tasks, we let $c_k^i = 0$ for all robots without loss of generality, and we drop the superscript $i$ for simplicity in this section.
For a mixed strategy, each robot $i\in\mathcal{I}$ has a probability $p_k$ of assigning itself to task $k\in\mathcal{S}$; this is a Bernoulli random variable with expectation,
\begin{equation}
    \mathbb{E}\left[\mathbbm{1}_k\right] = 1\cdot p_k + 0\cdot(1-p_k) = 1\cdot p_k.
\end{equation}
The number of robots assigned to task $k$ is a collection of Bernoulli random variables, which follows a Binomial distribution.
Thus, under Assumption \ref{smp:hysteresis}, the number of robots assigned to task $k$ after the assignment is,
\begin{equation} \label{eq:expectation}
    \mathbb{E}[N_k] = n_k + n_0\,p_k.
\end{equation}

\subsubsection{Idle Task Feasible, Homogeneous}
We first consider the case where assignment to the idle task is feasible, i.e., $0\in\delta$ (Definition \ref{def:mix-strat}).

\begin{lemma} \label{lma:can-idle}
    If $0\in\delta$, i.e., doing nothing is a feasible action, then each robot $i$ assigns itself to task $k\in\delta\setminus\{0\}$ with probability,
    \begin{equation} \label{eq:lma1}
        p_k = \frac{\gamma_k}{n_0}\Big(1 - s_k - \frac{n_k}{\gamma_k}\Big).
    \end{equation}
\end{lemma}

\begin{proof}
    By the definition of a mixed strategy, we set the marginal utility \eqref{eq:pik0} equal to zero, which implies,
    \begin{equation}
        E\big[\pi_{k0}\big] = \frac{\gamma_k - \mathbb{E}[N_k]}{\gamma_k} - s_k = 0.
    \end{equation}
    Substituting \eqref{eq:expectation} and solving for $p_k$ completes the proof.
\end{proof}

Note that Lemma \ref{lma:can-idle} also gives the values of $s_k$ where a mixed strategy exists.
Setting $p_k$ equal to $0$ or $1$ yields the signal range,
\begin{equation} \label{eq:signal-range}
    s_k \in \left[\, 1 - \frac{(n_0 + n_{k})}{\gamma_k},  1 - \frac{n_k}{\gamma_k} \right].
\end{equation}
Any signal less than the lower bound of \eqref{eq:signal-range} is a strictly dominant action, and thus Lemma \ref{lma:can-idle} does not apply.
Conversely, any signal larger than the upper bound of \eqref{eq:signal-range} is strictly dominated by remaining idle. In this case, Lemma \ref{lma:can-idle} yields a negative number, and we can consider the probability of assigning any robot to a dominated task $k$ as $p_k = 0$.

Lemma \ref{lma:can-idle} also describes when remaining idle is \textbf{not} a feasible action, i.e., when $0\not\in\delta$.
Namely, the probability of remaining idle is,
\begin{equation}
    p_0 = 1 - \sum_{k=1}^M \max(p_k, 0).
\end{equation}
Thus, $p_0 < 0$ implies that remaining idle is a strictly dominated strategy.

\subsubsection{Idle Task Infeasible, Homogeneous}

Next, we look at the case where remaining idle is dominated by some subset of all strategies, i.e., $0\not\in\delta$.

\begin{lemma} \label{lma:cant-stop}
When remaining idle is strictly dominated, the Nash equilibrium for task assignment is given by,
\begin{equation} \label{eq:big-matrix}
    \begin{bmatrix}
        \gamma_2 & -\gamma_1 & 0 & \dots & 0 \\
        \gamma_3 & 0  & -\gamma_1 & \dots & 0 \\
        \vdots & \vdots & \vdots & \ddots & \vdots\\
        \gamma_M & 0 & 0 & 0 & -\gamma_1\\
        1 & 1 & \dots & \dots & 1
    \end{bmatrix}
    \begin{bmatrix}
    p_1 \\ p_2 \\ p_3 \\ \vdots \\ p_K    
    \end{bmatrix}
    =
    \begin{bmatrix}
        b_2 \\ b_3 \\ \vdots \\ b_{M-1} \\ 1
    \end{bmatrix},
\end{equation}
where
\begin{equation}
    b_j = \frac{\gamma_1\,\gamma_j(s_1 - s_j) + n_1\gamma_j - n_j\gamma_1}{n_0},
\end{equation}
for indices $\{1, 2, \dots, M\} = \delta$.
\end{lemma}

\begin{proof}
By definition, the expected value of the marginal utility between any tasks $i,j\in\delta$ is equal to zero,
\begin{equation}
    \pi_{kj} = \frac{\gamma_k - E[N_k]}{\gamma_k} - s_k - \frac{\gamma_j - E[N_j]}{\gamma_j} + s_j = 0.
\end{equation}
Substituting \eqref{eq:expectation} and re-arranging terms yields,
\begin{equation} \label{eq:expectation-expanded}
    \gamma_k p_j - \gamma_j p_k
    =
    \frac{\gamma_k\gamma_j(s_k - s_j) + n_k\gamma_j-n_j\gamma_k}{n_0}.
\end{equation}
We arbitrarily select an index $k=1$, and let $j = 2, 3, \dots, M$.
This yields $M-1$ independent equations that constitute the first $M-1$ rows of \eqref{eq:big-matrix}.
The law of total probability implies that,
\begin{equation}
    \sum_{k\in\delta} p_k = 1,
\end{equation}
which is the final row of \eqref{eq:big-matrix}. 
Finally, \eqref{eq:big-matrix} is row-equivalent to the identity matrix; thus it is invertible.
This implies that the Nash Equilibrium exists and is unique.
\end{proof}

\subsection{Heterogeneous Systems}
Next, we consider the case where cost term $c_k^i$ in the utility function \eqref{eq:utility} is heterogeneous.
To generate the general solution, we first first partition the system into \textit{robot groups}.

\begin{definition} \label{def:types}
    A system of $N$ robots has $g\leq N$ \textit{groups}, where robots $i$ and $j$ are in the same group if and only if,
    \begin{equation}
        c_k^i = c_k^j,
    \end{equation}
    for all tasks $k = 1, 2, \dots, M$.
\end{definition}

We use the notion of a group to define useful vector quantities to track the state of the system,
\begin{align}
    \bm{n}_k &= \left[n_k^1,\, n_k^2,\, \dots,\, n_k^g \right]^{\intercal}, \\
    \bm{p}_k &= \left[p_k^1,\, p_k^2,\, \dots,\, p_k^g \right]^{\intercal}, \\
    \bm{p} &= [\bm{p}_1^{\intercal},\, \bm{p}_2^\intercal,\, \dots,\, \bm{p}_M^{\intercal}]^{\intercal},
\end{align}
where $\bm{n}_k$ and $\bm{p}_k$ store the number robots assigned to task $k$ and the probability of assignment to task $k$ for each group, respectively, and $\bm{p}$ stores the probability of all assignments from any group to any task.

Each idle robot in group $g$ has a probability $p_k^g$ of being assigned to task $k$, which is a Bernoulli random variable.
Similar to \eqref{eq:expectation}, the total number of robots that will be assigned to task $k$ follows a Binomial distribution with expectation,
\begin{equation} \label{eq:expectation-hetero}
    \mathbb{E}\big[|\bm{N}_k|\big] = |\bm{n}_k| + \sum_{i=1}^{g} \left(n_0^i\,p_k^i \right) = |\bm{n}_k| + \bm{n}_0\cdot\bm{p}_k,
\end{equation}
where $g$ is the number of groups (Definition \ref{def:types}), $n_0^{i}$ is the number of idle robots in group $i$, $|\cdot|$ sums the elements of a vector, and `$\cdot$' denotes the dot product.

\subsubsection{Idle Task Feasible, Heterogeneous.}

First, we consider the case where some idle robots can remain idle, i.e., $0\in\delta^i$ for some group $i$.

\begin{lemma} \label{lma:hetero-idle}
    For group $i$, if $0\in\delta^i$, i.e., doing nothing is not strictly dominated, then the distribution $\bm{p}_k$ is,
    \begin{equation} \label{eq:lma-htro-idle}
        \bm{n}_0\cdot \bm{p}_k = \gamma_k\Big(1 - s_k - c_k^i - \frac{|\bm{n}{_k}|}{\gamma_k}\Big).
    \end{equation}
\end{lemma}

\begin{proof}
    Substituting \eqref{eq:expectation-hetero} into the expectation of the marginal utility \eqref{eq:pik0} yields,
    \begin{equation}
        1 - \frac{|\bm{n}_k| + \bm{n}_0\cdot\bm{p}_k}{\gamma_k} - s_k - c_k^i = 0,
    \end{equation}
    which completes the proof.
\end{proof}

An important aspect of Lemma \ref{lma:hetero-idle} is that, for a task $k$, it yields $g$ equations for the distribution $\bm{p}_k$.
This gives us information about the supports $\delta^1, \delta^2, \dots, \delta^g$ that satisfy the premise of Lemma \ref{lma:hetero-idle}, which we present in our next result.

\begin{lemma} \label{lma:hetero-idle-elim}
    For task $k$, only the groups that satisfy,
    \begin{equation}
        \arg\min_{i} \left\{ c_k^i ~:~ n_0^i > 0\right\},
    \end{equation}
    satisfy the premise of Lemma \ref{lma:hetero-idle}.
\end{lemma}

\begin{proof}
    Let $\mathcal{G} \subseteq \{i = 1, \dots , g : n_0^i > 0\}$ be a set of groups that that satisfy the premise of Lemma \ref{lma:hetero-idle}.
    First, any group $i$ such that $n_0^i = 0$ cannot satisfy Lemma \ref{lma:hetero-idle}, as $p_k^i = 0$ for all $k = 1, 2, \dots, M$ under Assumption \ref{smp:hysteresis}.
    Then, \eqref{eq:lma-htro-idle} implies,
    \begin{align*}
        \bm{n}_0\cdot\bm{p}_k
        = \gamma_k\left(1 - s_k - c_k^i - \frac{|\bm{n}_k|}{\gamma_k} \right) \quad \forall i\in\mathcal{G},
    \end{align*}
    which implies,
    \begin{equation}
        c_k^i = c_k^j \quad \forall i,j\in\mathcal{G},
    \end{equation}
    which implies Lemma \ref{lma:hetero-idle-elim}.
\end{proof}

Thus, Lemma \ref{lma:hetero-idle-elim} implies a straightforward solution: i) for each task, eliminate all groups with idle robots except those that minimize the cost, and ii) solve the remaining linear system of equations.

\subsubsection{Idle Task Infeasible, Heterogeneous.}

Finally, we derive the Nash equilibrium for the general heterogeneous robot assignment problem when remaining idle is infeasible.
\begin{lemma} \label{lma:hetero}
    When remaining idle is infeasible, the Nash equilibrium for task assignment has the form,
    \begin{equation*}
        \begin{bmatrix}
             \Lambda  \\
             \Lambda \\
             \vdots   \\
             \Lambda \\
            I_{g\times g}
            \otimes I_{N\times N}
        \end{bmatrix}
       \bm{p}
       =
       \begin{bmatrix}
           b^1 \\ b^2 \\ \vdots \\ b^g \\ \bm{1}^g
       \end{bmatrix},
\end{equation*}
where $\bm{1}^M$ is an $M$-length vector of ones, $\otimes$ is the Kronecker product, $I_{N\times N}$ is the $N\times N$ identity matrix,
\begin{equation} \label{eq:Lambda}
 \Lambda = 
    \begin{bmatrix}
        \gamma_2\bm{n}_0^{\intercal} & -\gamma_1\bm{n}_0^{\intercal} & 0 & 0 & \dots & 0 \\
        \gamma_3\bm{n}_0^{\intercal} & 0 & -\bm{n}_0^{\intercal} & 0 & \dots & 0 \\
        \gamma_4\bm{n}_0^{\intercal} & 0 & 0 & -\bm{n}_0^{\intercal} & 0 & \dots \\
        \vdots & 0 & 0 & 0 & \ddots & 0 \\
        \gamma_M\bm{n}_0^{\intercal} & 0 & 0 & 0 & \dots & -\bm{n}_0^{\intercal} 
    \end{bmatrix},
\end{equation}
is an $(M-1)\times(N\cdot M)$ matrix,
\begin{equation} \label{eq:b}
b^i = 
\begin{bmatrix}
\gamma_1\gamma_2\left((s_1 + c_1^i) - (s_2 + c_2^i) + \frac{|\bm{n}_1|}{\gamma_1} - \frac{|\bm{n}_2|}{\gamma_2}\right) \\
\gamma_1\gamma_3\left((s_1 + c_1^i) - (s_3 + c_3^i) + \frac{|\bm{n}_1|}{\gamma_1} - \frac{|\bm{n}_3|}{\gamma_3}\right) \\
\vdots \\
\gamma_1\gamma_M\left((s_1 + c_1^i) - (s_M + c_M^i) + \frac{|\bm{n}_1|}{\gamma_1} - \frac{|\bm{n}_M|}{\gamma_M}\right)
\end{bmatrix},
\end{equation}
and $|\bm{n}_k|$ is the sum of the elements of $\bm{n}_k$.
\end{lemma}

\begin{proof}
    The marginal utility for robot $i$ to select task $k$ over task $j$ is,
    \begin{equation} \label{eq:marginal-utility-hetero}
        \pi_{kj}^i = \left(\frac{|\bm{N}_j|}{\gamma_j} - \frac{|\bm{N}_k|}{\gamma_k}\right) + (s_j + c_j^i) - (s_k + c_k^i).
    \end{equation}
    Setting the expectation equal to zero yields,
    \begin{align}
        \mathbb{E}\big[\pi_{kj}^i\big] =& 
        \frac{|\bm{n}_j| + \bm{n}_0\cdot\bm{p}_j}{\gamma_j} - \frac{|\bm{n}_k| + \bm{n}_0\cdot\bm{p}_k}{\gamma_k}
        \notag \\
        &+ (s_j + c_j^i) - (s_k + c_k^i) = 0,
    \end{align}
    which implies,
    \begin{equation*}
        \bm{n}_0\cdot\left(\frac{\bm{p}_j}{\gamma_j} - \frac{\bm{p}_k}{\gamma_k}\right)  = (s_k + c_k^i) - (s_j + c_j^i) + \left(\frac{|\bm{n}_k|}{\gamma_k} - \frac{|\bm{n}_j|}{\gamma_j}\right).
    \end{equation*}
    We arbitrarily select $j=1$, which yields $(M-1)\cdot N$ equations for $k = 2, 3, \dots, M$  and $i = 1, 2, \dots N$,
    \begin{equation*}
        \bm{n}_0\cdot\left(\frac{\bm{p}_1}{\gamma_1} - \frac{\bm{p}_k}{\gamma_k}\right)  = (s_k + c_k^i) - (s_1 + c_1^i) + \left(\frac{|\bm{n}_k|}{\gamma_k} - \frac{|\bm{n}_1|}{\gamma_1}\right),
    \end{equation*}
    where multiplying by $\gamma_1\gamma_k$ yields makes the left and right-hand side are equal to $\Lambda$ and $b_k$, respectively.
    The remaining $N$ rows come from the definition of a probability distribution, i.e.,
    \begin{equation}
        \sum_{k\in\mathcal{S}} p_k^i = 1 \quad \forall i=\{1, 2, \dots, g\},
    \end{equation}
    which completes the proof.
\end{proof}

Lemma \ref{lma:hetero} yields the mixed Nash equilibrium for each task under the premise that the supports $\delta^1, \delta^2, \dots, \delta^g$ are suitably chosen.
Like Lemma \ref{lma:hetero-idle}, we can derive information about feasible supports for each group by ensuring the matrix in Lemma \ref{lma:hetero} is full-rank.

\begin{lemma} \label{lma:reduction}
    The collection of supports $\delta^1, \delta^2, \dots, \delta^g$ that satisfy the premise of Lemma \ref{lma:hetero} yield a mixed Nash equilibrium that satisfies,
    \begin{equation}
        \begin{bmatrix}
            \Lambda \\
            A
        \end{bmatrix}
        \bm{p}
        =
        \begin{bmatrix}
            \bm{b} \\
            \bm{1}
        \end{bmatrix},
    \end{equation}
    where $A$ is an appropriately sized binary matrix that contains a subset of the $I_{g\times g}\otimes I_{N\times N}$ terms in Lemma \ref{lma:hetero}, and $\bm{1}$ is an appropriately sized vector of ones.
\end{lemma}

\begin{proof}
    We prove Lemma \ref{lma:reduction} in two steps.
    First, let the first $(M-1)$ rows not be row-equivalent to the form of $\Lambda$.
    This implies that at least one task $k = 1, 2,$ in the support has a zero probability of assignment, which contradicts our premise.
    
    Next, let the first $ Q > (M-1)$ rows have the form of $\Lambda$.
    This implies that there are $Q-(M-1)$ linearly dependent rows.
    If two rows share the same cost, i.e., $c_{r}^u = c_{r}^v$,
    then by the definition of a mixed strategy (Definition \ref{def:mix-strat}),
    \begin{equation}
        \frac{\gamma_r - \mathbb{E}[n_r]}{\gamma_r} - s_r - c_r^u
        = \frac{\gamma_r - \mathbb{E}[n_r]}{\gamma_r} - s_r - c_r^v,
    \end{equation}
    which implies that $p_r^u = p_r^v$.
    Thus, the number of robots available for assignment to task $r$ is $n_0^u + n_0^v$, which we can write as a single line while maintaining the diagonal structure of $\Lambda$.
    If two linearly dependent rows do not share the same cost, i.e., $c_{r}^u < c_{r}^v$, then group $v$ is strictly dominated by $u$, and $p_r^u = 0$.
    This implies that we must remove the row corresponding to $p_r^u$, which completes the proof.
\end{proof}

The proof of Lemma \ref{lma:reduction} has a structure that immediately yields an interesting property of our assignment algorithm, which we present next.

\begin{remark} \label{rmk:capable}
    The most capable idle robots are \textbf{always} assigned to a task before less capable robots.
\end{remark}

\begin{proof}
    In the homogeneous case, all robots are equally capable and Remark \ref{rmk:capable} holds trivially.
    For the heterogeneous case, this follows directly from Lemma \ref{lma:hetero-idle-elim} for the case where remaining idle is feasible, otherwise it follows from the reduction of strictly dominated robots in Lemma \ref{lma:reduction}.
\end{proof}

Thus, our framework generates the \textbf{unique} mixed Nash equilibrium that optimizes the allocation of robots to goals, and robots are allocated in a way where more capable robots are \textit{always} allocated first.
However, finding the mixed Nash equilibrium requires knowledge of the appropriate supports $\delta^1, \delta^2, \dots, $ for each group.
We present an efficient solution to generate the support using \textit{iterated elimination of dominant strategies} \cite{Morris2001GlobalApplications}, which we present in Algorithm \ref{alg:elimination} and Theorem \ref{thm:algorithm}.

\begin{theorem} \label{thm:algorithm}
    The optimal allocation of robots to tasks is achieved by Algorithm \ref{alg:elimination}.
\end{theorem}

\begin{proof}
    Theorem \ref{thm:algorithm} holds from Lemma \ref{lma:can-idle} in the case that idling is a feasible action, i.e., $0\in\delta$.
    Otherwise, it holds by Lemmas \ref{lma:hetero} and \ref{lma:reduction}, by eliminating the rows corresponding to strictly dominated assignments.
\end{proof}

\begin{algorithm}[h] 
    \Input{$\bm{s}$, $\{\bm{c}^1,\dots, \bm{c}^{g}\}$, $\{\bm{n}_0, \dots, \bm{n}_M\}$}
    \Output{$\bm{p}$}
    remove all groups $i=1,2,\dots, g$ where $n_0^i = 0$\;
    initialize all elements $\bm{p}\to\bm{0}$\;
    \For{$i=1,2,\dots,g$}{
        $\delta^i \gets \arg\min_{j\in\mathcal{S}}\{c_j^i, n_0^i\}$ (initialize support)\;
        \If{$\delta^i \neq \emptyset$}{
            $p_0^i = -1$ (Initialize to invalid value)\;
        }
    }
    \tcp{Assign Tasks where Idle is Feasible}
    \Do{$\{i\in\mathcal{I} ~:~ p_0^i < 0\}\neq \emptyset$}
    {
        $\bm{p}^i \gets \max\left\{0,\, \eqref{eq:lma-htro-idle}\right\}$ (assignment with idle)\; 
        $\delta^i = \delta^i \setminus\{k : p_k^i \leq 0\}$ (elimination)\;
        $p_0^i \gets \left(1 - \sum_{k\in\delta^i}\, p_k\right)$\;
    }
    \tcp{Assign Tasks where Idle is Infeasible}
    \Do{$\exists\, i ~:~\sum_k p_k^i \neq 1$}{
        \For{$i\in\mathcal{I}$}{
            $\delta^i = \delta^i \setminus \{k : p_k^i\leq 0\}$ (elimination)\;
            $\bm{b}^i \gets$ \eqref{eq:b}\;
            }
        $\Lambda \gets$ \eqref{eq:Lambda}\;
        $\bm{p} \gets \max\left\{\Lambda^{-1}\bm{b},\, \bm{0}\right\}$ (Assignment, Lemma \ref{lma:hetero-idle-elim})\;
    }
    \textbf{return:} $\bm{p}$;
    \caption{Iterated elimination of dominant strategies for the multi-robot task allocation problem.}
    \label{alg:elimination}
\end{algorithm}

\begin{remark}
        Algorithm 1 finds the optimal assignment with $O(M\cdot g)$ computational complexity.
\end{remark}

\begin{proof}
By Theorem \ref{thm:algorithm}, Algorithm \ref{alg:elimination} finds the mixed Nash Equilibrium.
It achieves this in 3 steps, (i) iterate through $N$ robots to find the idle robots with the minimum cost for each task, (ii) perform iterated elimination in the case where the idle task is feasible, and (iii) perform iterated elimination when the idle task is infeasible.
The first part has $O(g)$ complexity as it creates the support $\delta^i$ that assigns only the most capable robots to each task.
The second part has $O(M)$ complexity, since at least one task is removed each iteration and there is only one robot group per task.
The third part has $O(M\cdot g)$ complexity; at least one task is eliminated in the while loop at each iteration, and the for loop loops through all $g$ groups at most.
This results in a final complexity of $O(g+M+M\cdot g)$, which simplifies to $O(M\cdot g)$ and completes the proof.
\end{proof}

Thus, with Algorithm \ref{alg:elimination}, we present a computationally efficient method to allocate robots to dynamic tasks based on the weighting factor $\lambda^k$ for each task, the cost $c_k^i$ for each robot-task pair, and the current state of each task $s_k$.
By efficiently selecting the parameters $\lambda_k$ and $c_k^i$ in the objective function, one can achieve an efficiently allocate where the robots idle at a central location while monitoring the completion of dynamic tasks.
This makes it straightforward for robots to be added to or remove from the system, as this only changes the vector $\bm{n}$ in \eqref{eq:big-matrix}.
Similarly, if a new task is added or the importance of an existing task changes quickly, this directly affects the allocation of robots through \eqref{eq:big-matrix}.
Our working Assumptions \ref{smp:information} and \ref{smp:hysteresis} suggest that our approach works best in systems with some notion of a centralized depot--where idle robots monitor the values of $\bm{n}$ and $\bm{s}$--and that robots assigned to tasks return to periodically.
We present two possible scenarios in the following section.

\section{Simulation Results} \label{sec:sim}

To validate the performance of our Algorithm \ref{alg:elimination}, we performed a collection of Matlab simulations.
In Section \ref{sec:colony}, we present a \textit{colony maintenance problem}, where robots must gather energy sources in the environment while also moving cargo to the colony that is periodically delivered.
We also remove 50\% of the robots halfway through the simulation to demonstrate our algorithm's robustness to the addition and removal of robots to the system.

In Section \ref{sec:colony}, we present a \textit{persistent monitoring problem}, where $N$ robots must periodically visit a collection of $M > N$ nodes to collect information that accumulates over time \cite{hall2023bilevel}.
We select a linear mapping the state of each node (information quantity) to the corresponding signal read by each robot, and we set the cost $c_k^i$ equal to the distance between robot $i$ and node $k$.
This demonstrates how our system performs for the most heterogeneous system, i.e., the group size is $1$ robot.

In both simulations, \textbf{collision avoidance} is guaranteed for all robots using a cooperative control barrier function (CBF) \cite{Ames2019ControlApplications}.
Namely, each task produces a reference velocity $\bm{v}_{ref}$, which steers the robot toward their goal at their maximum speed.
We then generate the actual velocity signal by solving Problem \ref{prb:cbf} at each time step.

\begin{problem} \label{prb:cbf}
For each of robot at each time step $t_k$, find the control input for each robot $i$ that optimizes,
    \begin{align*}
        \min_{\bm{v}_i(t_k)} & \frac{1}{2}||\bm{v}_i - \bm{v}_{ref}||^2 \\
        \text{subject to: }&\\
        ||\bm{v}_i(t_k)|| &\leq v_{\max}, \\
        \bm{p}_i^{\intercal}\bm{v}_i &+ (||\bm{p}_i||^2 - R_o^2) \leq 0, \\
        (\bm{p}_i-\bm{p}_j)^{\intercal}\bm{v}_i - \frac{1}{2}v_{\max}||\bm{p}_i - \bm{p}_j|| &+ (||\bm{p}_i - p_j||^2 - r^2) \leq 0,
    \end{align*}
    where the robots have single-integrator dynamics,  the first constraint enforces the control bounds, the second constraint keeps the robot within a distance $R_0$ of the center of the domain, and the third equation guarantees collision avoidance between robots.
\end{problem}

\subsection{Colony Maintenance} \label{sec:colony}

For the first example, we considered a collection of $N=12$ robots in an autonomous colony and $K=2$ tasks, energy harvesting and cargo transporting.
Furthermore, we randomly remove half of the robots approximately halfway through the simulation to demonstrate our approach's resilience to robot failure.
We describe the geometry of the simulation next, then discuss each of the behaviors corresponding to the two tasks.
The parameters used for this example are presented in Table \ref{tab:params}.

\begin{table}[ht]
    \centering
    \caption{Robot, environment, and task parameters use for the Matlab simulation.}
    \label{tab:params}
    \begin{tabular}{cccccccc}
        $h$ & $r$ & $R_o$ & $R_i$ & $E_{source}$ & $E_{drain}$ & $\gamma_1$ & $\gamma_2$ \\ \toprule
        5 m & 0.25 m & 30 m & 5 m & 4 J & 0.1 J s & 12 & 7.2
    \end{tabular}
\end{table}

The domain for the \textit{colony maintenance} task is an annulus with outer radius $R_0$ and inner radius $R_i$ centered at the origin.
The inner disk represents the colony, where robots idle, recharge, and deliver energy sources and cargo.
The energy sources and cargo are placed within the domain, and we consider them to be delivered when a robot carrying them reaches a distance $R_i$ of the origin.
To assign robots to tasks (idle, energy harvesting, cargo transporting), we use Algorithm \ref{alg:elimination} at each time step to determine the probabilities $p_0, p_1, p_2$.
Finally, each robot calculates a uniformly distributed random number $P^i\in[0, 1]$ and determines its assigned task by comparing against the distribution, i.e.,
\begin{equation}
    \begin{aligned}
        0 &\leq P^i \leq p_0 &&\implies \text{idle}, \\
        p_0 &\leq P^i \leq p_0 + p_1 &&\implies \text{energy harvesting}, \\
        p_0 + p_1 &\leq P^i \leq 1 &&\implies \text{cargo transporting}.
    \end{aligned}
\end{equation}
To complete either task, the robot must return to the colony with either a piece of cargo or an energy source.
When this occurs, the robot switches back to the idle task and repeats the assignment process outlined above.
Next, we discuss the dynamics of each task and the corresponding signals $s_1(t), s_2(t)$.

\textbf{Energy Harvesting} involves the robot performing a random walk to search the domain for energy sources.
We model the energy dynamics of the colony as a linear function,
\begin{equation} \label{eq:energy-harvesting-dynamics}
\begin{aligned}
    \dot{E}_c(t) =& -E_{drain}
    + \mathbbm{1}_E E_{source} - \mathbbm{1}_C E_{charge},
\end{aligned}
\end{equation}
where $E_{drain}$ is a constant leakage rate, $\mathbbm{1}_E$ is an indicator variable that is $1$ when an energy source is delivered and $0$ otherwise, $E_{source}$ is the energy contents of the energy sources in the environment, $\mathbbm{1}_C$ is an indicator variable that is $1$ when a robot is recharging and $0$ otherwise, and $E_{charge}$ is the energy cost of recharging a robot.
Each robot also consumes energy at a rate of,
\begin{equation} \label{eq:robot-energy}
    \dot{E}_{robot} = -0.1\,E_{drain} \frac{||\bm{v}||}{v_{\max}} + \mathbbm{1}_c E_{charge},
\end{equation}
which is an order of magnitude less energy than the leakage rate of the colony.
The indicator variable $\mathbbm{1}_C E_{charge}$ sets $E_{robot}$ back to zero, which restores any energy used by the robot's motion.
Each of the parameters is given in Table \ref{tab:params}, and finally we define the task signal as,
\begin{equation}
    s_1(t) = \frac{E(t)}{E_{\max}},
\end{equation}
which is a linear mapping to the normalized energy of the colony.
Finally, we initialize $E_c(t=0) = \frac{1}{2}E_{\max}$ to avoid an initial transient where all robots are idle because $s_1(t) \approx 1$.

To perform the task, we use a random walk similar to \cite{krieger2000call} as an illustrative example.
First, we assume each robot has a finite sensing distance $h$.
To implement the random walk, the robot selects a random heading angle in $[0, 2\pi]$, and generates a target position at a distance $h$ to move towards.
If the target position is outside of the  domain, we project the target position onto the domain boundary.
After reaching the target position, the robot randomly generates another heading angle, which yields a new target position.
This process is repeated until the robot comes within a distance $h$ of an energy source in the domain.
Once this occurs, the robot moves to the position of the energy source, then follows a straight-line trajectory back to the origin.
Once the robot reaches the colony (i.e., a distance $R_i$ from the origin), the colony energy is increased by a constant $E_{source}$, the robot switches to an idle task, then immediately updates its assignment.
If the robot re-assigns itself to the energy harvesting task again, it attempts to return to the location of the energy source it previously retrieved, with some Gaussian noise scaled by the distance traveled.
Similar to \cite{Krieger2000}, this simulates the impact of stochasticity in the robot's sensors and actuators.

\textbf{Cargo Transporting} involves the robot transferring cargo from a known depot location to the colony.
The amount of cargo $c(t)$ is a discrete number between $0$ and $c_{\max}$.
To deliver cargo, robots travel in a straight line toward the known depot location, wait for a small time delay, then travel back toward the origin on a straight-line trajectory.
This simulates the robot automatically acquiring and delivering supplies to the colony, e.g., scientific equipment, from a known drop-off location.
We define the cargo signal as a linear mapping,
\begin{equation}
    s_2(t) = 1 - \frac{c(t)}{c_{\max}},
\end{equation}
where $c_{\max} = 10$ is the cargo capacity of the depot.
For this task, robots travel in a straight line to the depot location, wait for a small time delay, then return along a straight-line path to the colony.
This simulates the robots autonomously delivering supplies to the colony, e.g., scientific equipment or maintenance supplies.

In our simulation, cargo is delivered at two different time instants,
\begin{align*}
    t_1 = 0.25\cdot t_f &= 120\text{ seconds}, \\
    t_2 = 0.375\cdot t_f &= 225\text{ seconds},
\end{align*}
where $t_f = 600$ s is the final time step of the simulation.
For each delivery, we set the cargo quantity $c(t_1) = c(t_2) = c_{\max} = 10$ units.
Both times the cargo is located in a cargo depot located at $(20, 0)$, which just over half the distance from the colony to the domain boundary (Table \ref{tab:params}).

\begin{figure*}[t]
    \begin{center}
    \hfill
    \includegraphics[width=0.27\linewidth]{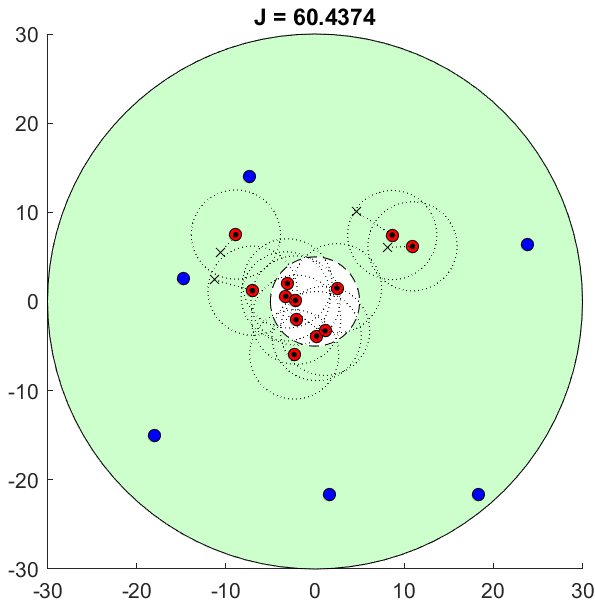}
    \hfill
    \includegraphics[width=0.27\linewidth]{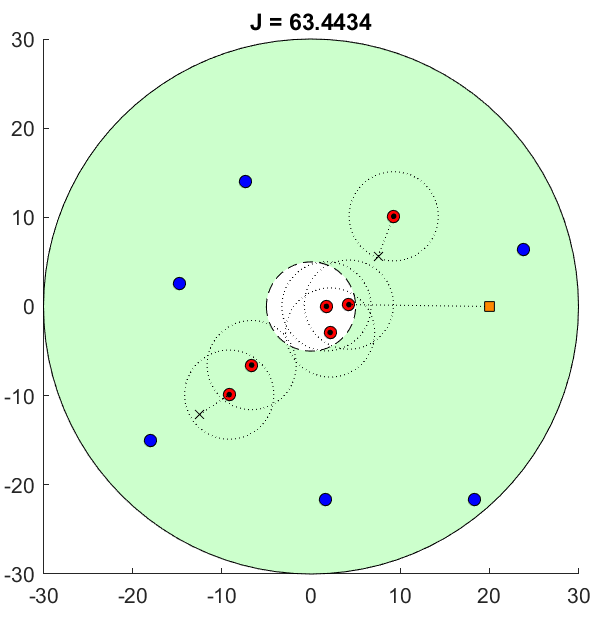}
    \hfill
    \includegraphics[width=0.27\linewidth]{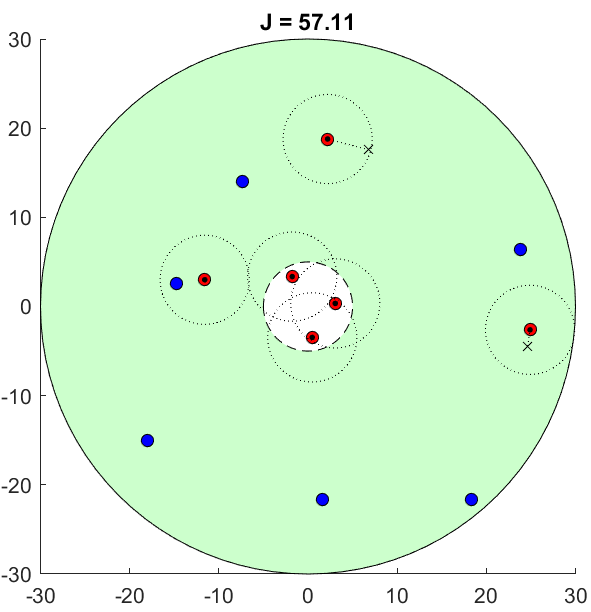}
    \hfill
    \end{center}
    \caption{The simulation state at (left to right) $t=100, 200, 400$ seconds; these show the initial configuration, configuration after half the robots are removed, and final configuration of the robots.}
    \label{fig:colony frames}
\end{figure*}

\textbf{Human Interference} occurs halfway through the simulation, where $N=6$ robots are removed from the simulation.
This represents a human colonist recruiting these robots to help them with another task, such as gathering samples or remote research.
We randomly select 6 robots to remove at $t_r = 0.5(t_1 + t_2) = 172.5$ seconds, which is exactly halfway between cargo deliveries.
If an idle robot is removed, we simply reduce the total number of robots $N$ by one; if a foraging or cargo carrying robot is removed, we also reduce $n_1$ or $n_2$ by $1$, respectively.
In the case of a cargo carrying robot, we also replace the robot's carried cargo at the depot by increasing $c(t)$ by the number of robots recruited away from task $2$.

Next we present the results of an individual simulation run, followed by aggregate results from a round of Monte Carlo simulations with random energy source placements.
The state of the system is presented at $4$ different time steps in Fig. \ref{fig:colony frames}.
These demonstrate the initial behavior of the colony; initially $5$ of the $12$ robots are actively harvesting energy, while the remaining $7$ robots are idle.
At $t_1 = 120$ s 10 units of cargo are delivered at the coordinate $(20, 0)$ as denoted by the orange square, and at $t_r = 172.5$ s $6$ of the robots are recruited by a human and removed from the system.
The second snapshot at $t=200$ s shows how the remaining $6$ robots are allocated, $3$ are foraging for energy sources, $1$ is transporting cargo, and $2$ remain idle.
The final snapshot at $400$ s shows the final state of the system after all cargo has been delivered.

Figs. \ref{fig:energy-task} and \ref{fig:cargo-task} demonstrate the performance of our approach on the energy harvesting and cargo delivering tasks, and Fig. \ref{fig:num-robots} shows the allocation of robots to tasks..
In Fig. \ref{fig:energy-hist}, the solid red line corresponds to the energy level remaining in the colony as described by the energy dynamics \eqref{eq:energy-harvesting-dynamics}.
The dashed black line is the total system energy, which subtracts the energy used by each robot via \eqref{eq:robot-energy} from the colony energy.
Each vertical step in the red curve corresponds to an energy source being delivered, which adds $4$ J of energy to the colony.
At the same time instant, the harvesting robots recharge their batteries.
The resulting sawtooth shape highlights the importance of the $0.1$ factor in the robot energy consumption \eqref{eq:robot-energy}; since the robots return more energy than they return to the system, the colony energy experiences an overall increase in energy at each of these time steps.

\begin{figure}[ht]
    \centering
    \includegraphics[width=\linewidth]{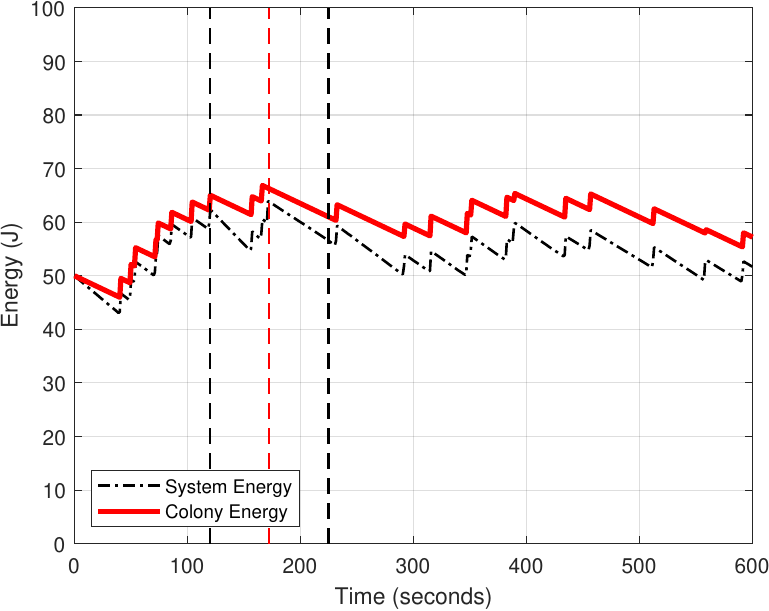}
    \caption{Energy level for the colony (red) and entire system (black, dashed) during the colony maintenance simulation.}
    \label{fig:energy-task}
\end{figure}

Fig. \ref{fig:cargo-task} shows how much cargo is available at the depot for the duration of the simulation.
Initially the depot is empty until $t_1=120$ s, where $10$ units of cargo are delivered.
The steep drop around $t_1=160$ s indicates that $8$ units of cargo are delivered relatively quickly.
Note that the cargo is approximately $15$ m from the edge of the colony, and the robots have a maximum speed of $1$ m/s, so this indicates that all $8$ robots began to deliver cargo almost immediately.
The dashed black line at $t_r=172.5$ denotes where half of the robots in the system are randomly removed, and we can see the cargo count goes to $0$ just before the second delivery at $t_2=220$ seconds.

\begin{figure}[ht]
    \centering
    \includegraphics[width=0.8\linewidth]{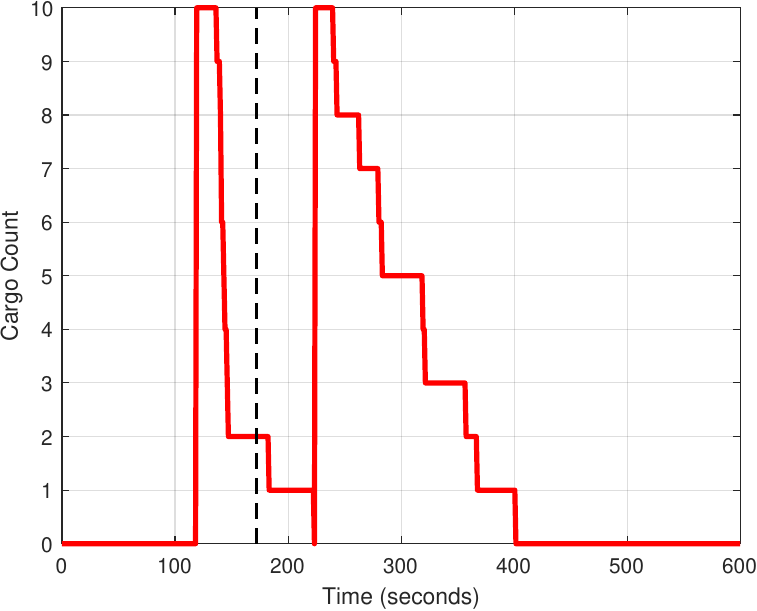}
    \caption{Amount of undelivered cargo at each time step, the vertical steps correspond to the arrival and delivery of cargo, respectively.}
    \label{fig:cargo-task}
\end{figure}

Finally, Fig \ref{fig:num-robots} shows a stacked chart for the total robots available and the number of robots assigned to each task.
Due to the low initial energy of the colony, initially $8$ of the $12$ robots were assigned to energy harvesting.
That number gradually decreased as the energy levels increased, until $8$ robots are assigned to deliver cargo at $t_1 = 120$ seconds.
Near the end of the simulation, after half the robots are removed and all the cargo is delivered, the system settles down to having just $3-4$ robots harvesting energy while the rest remain idle.
Overall, Figs. \ref{fig:energy-task}, \ref{fig:cargo-task}, and Fig. \ref{fig:num-robots} demonstrate that the robots were successfully able to deliver $20$ pieces of cargo while maintaining a positive net energy level while having half the team removed partway through.

\begin{figure}[ht]
    \centering
    \includegraphics[width=0.8\linewidth]{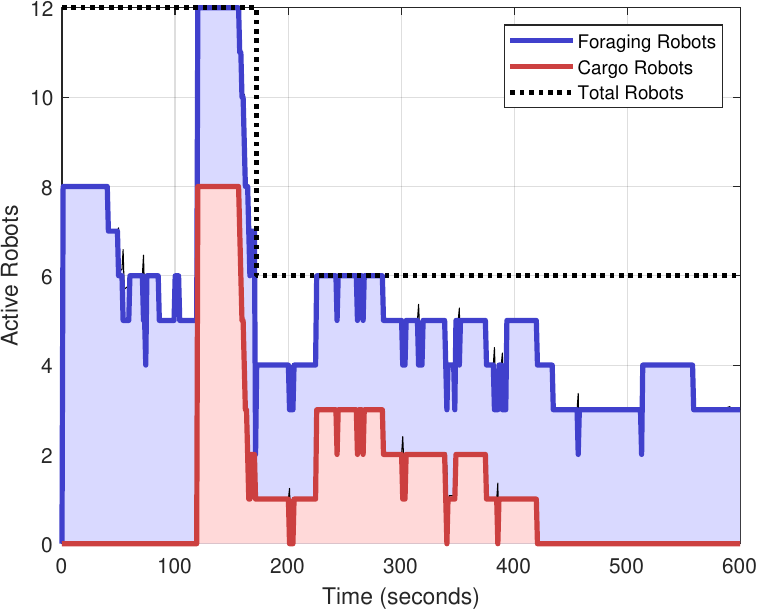}
    \caption{Caption}
    \label{fig:num-robots}
\end{figure}

Next, we demonstrate the performance of our algorithm with a series of Monte Carlo experiments.
We performed $100$ experiments where the location of energy sources throughout the environment was randomized.
Fig. \ref{fig:energy-hist} shows the distribution of energy at the final time step, while Fig. \ref{fig:cargo-hist} shows how many seconds it took to successfully deliver all the cargo.

The energy histograms in Fig. \ref{fig:energy-hist} show two interesting features.
First, the most frequent outcome is the $45-50$ J energy bin.
As with the previous scenario, we initialized the system with $50$ J of energy.
This demonstrates that for a plurality of cases, the system did not lose more than $5$ J of energy throughout the entire simulation.
The other interesting feature is that only $1$ of the $100$ cases reached an energy level of $0$, which caused the simulation to fail. 
While we don't provide any strict guarantees on maintaining the signal value above zero, this demonstrates the resilience of our approach to both the removal of random robots stochasticity in the location of energy sources.

The cargo delivering histogram in Fig. \ref{fig:cargo-hist} demonstrates how many seconds it takes to successfully delivered all of the cargo after $t = 120$ seconds, when the first load of cargo arrives.
Over the $100$ trials, there were only $9$ cases where all of the cargo was not delivered by $t_f = 600$ seconds.
Surprisingly, none of these is the case where the system ran out of energy.
Instead, this was caused by deadlock between the cargo delivering robots and other robots leaving the colony to harvest energy.
This is a known problem with our CBF-based collision avoidance controller (Problem \ref{prb:cbf}), and is not attributable to our assignment algorithm.
For the remaining $91$ cases, all cargo is within $385$ seconds, with median of $287$ seconds.

\begin{figure}
    \centering
    \includegraphics[width=0.8\linewidth]{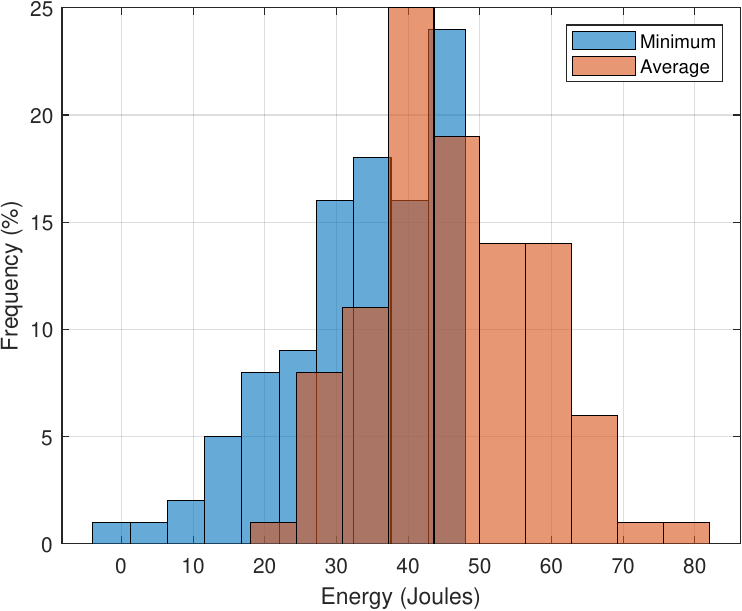}
    \caption{Caption}
    \label{fig:energy-hist}
\end{figure}

\begin{figure}
    \centering
    \includegraphics[width=0.8\linewidth]{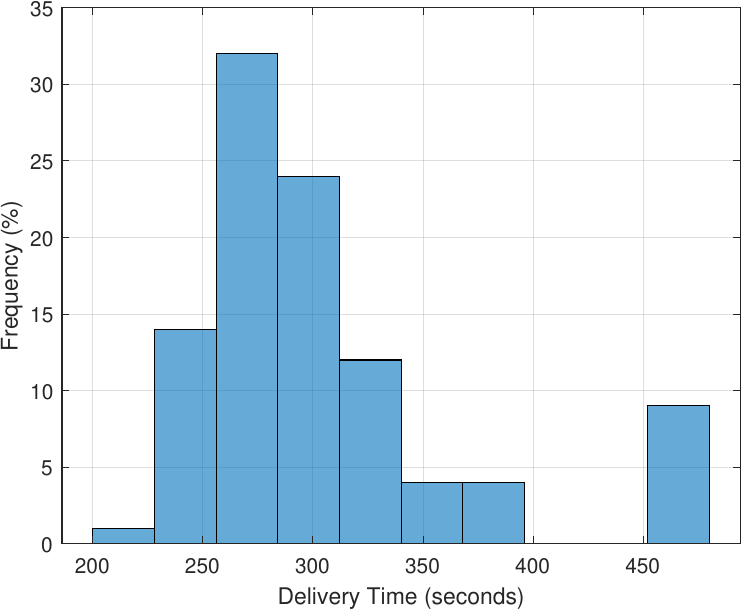}
    \caption{Caption}
    \label{fig:cargo-hist}
\end{figure}

\subsection{Persistent Monitoring} \label{eq:monitoring}

\begin{figure*}[t]
    \begin{center}
    \hfill
    \includegraphics[width=0.27\linewidth]{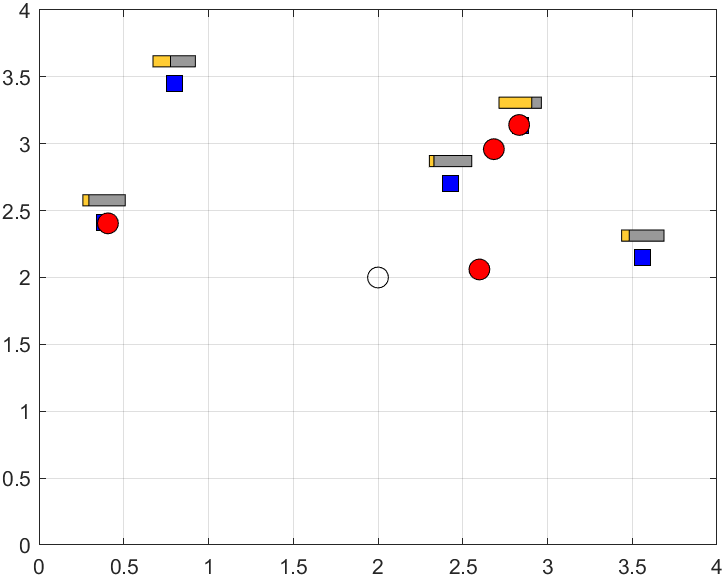}
    \hfill
    \includegraphics[width=0.27\linewidth]{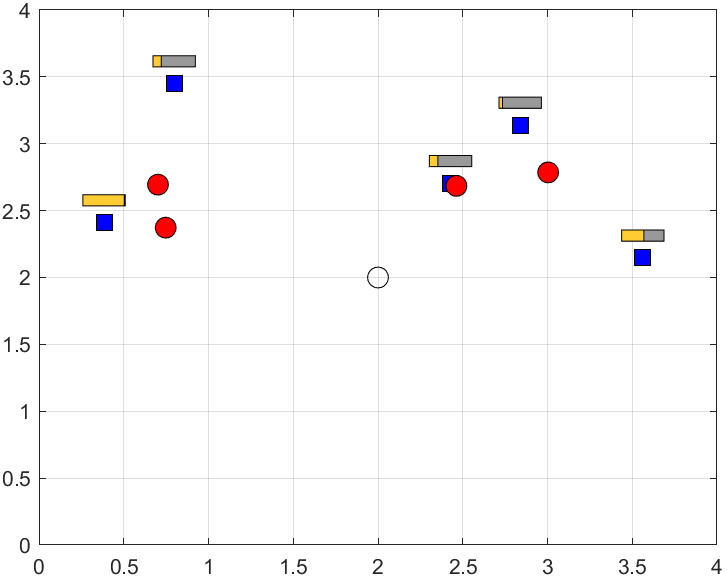}
    \hfill
    \includegraphics[width=0.27\linewidth]{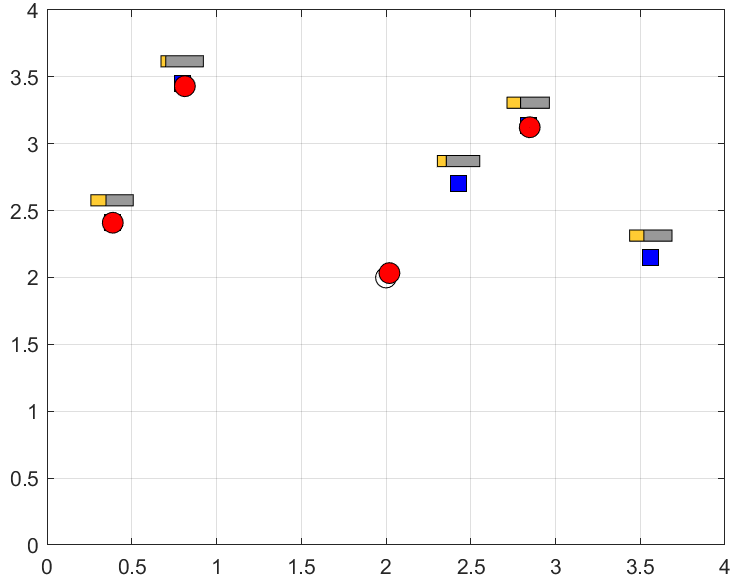}
    \hfill
    \end{center}
    \caption{The simulation state at (left to right) $t=300, 585, 985$ seconds. Robots are red circles, nodes are blue squares, the grey circle at $(2, 2)$ is the idle location, and the yellow bars show $\frac{R_k(t)}{R_{\max}}$ for each node.}
    \label{fig:pm_frames}
\end{figure*}

Next, we consider a persistent monitoring problem where $N=4$ robots must monitor $M=5$ nodes.
Each node $k$ accumulates information following linear dynamics \cite{hall2023bilevel},
\begin{equation}
    \dot{R}_k = A - B\sum_{i=1}^{N}\mathbbm{1}_k^i,
\end{equation}
where $A > 0$ is the rate that nodes accumulate information, $B > 0$ is the rate that robots collect information, and $\mathbbm{1}_i$ indicates whether robot $i$ is within range $r_0$ of node $k$.
We enforce minimum and maximum information constraints,
\begin{equation}
    0\leq R_k(t) \leq R_{\max}.
\end{equation}
Our numerical values for the simulation parameters are given in Table \ref{tab:pm-params}.

\begin{table}[ht]
    \centering
    \caption{The parameters corresponding to node information, domain size, and robot movement for the persistent monitoring problem.}
    \begin{tabular}{cccccc}
        $A$ & $B$ & $R_{0}$ & $r$ & $D$ & $v_{\max}$ \\\toprule
        $0.75$ & $2$ & $1$ & $0.04$ m & $4$ m & $4$ m/s 
    \end{tabular}
    \label{tab:pm-params}
\end{table}

We initialize each node to a value of $R_k(t=0) = 0$, i.e., there is no information available.
To map the node state to  a signal, we use a linear function,
\begin{equation}
    s_k(t) = 1 - \frac{R_k(t)}{R_{\max}},
\end{equation}
which satisfies Assumption \ref{smp:monotonicity}.
We also introduce a heterogeneous cost, which is the distance between each robot-node pair, so for each $i = 1,2,\dots N,$ and $m=1, 2,\dots,M$,
\begin{equation}
    C_m^n = \sqrt{||\bm{p}_m - \bm{p}^n||^2},
\end{equation}
where $\bm{p}_m$ is the position of node $m$, and $\bm{p}^n$ is the position of robot $n$.
Finally, we normalized the cost $C_m^N$ by the size $R_0$ of the domain.
This ensures that the cost $C_m^n \leq 1$, which means that there is some values of $s_m \geq 0$ such that $p_m^n > 0$ by \eqref{eq:utility}.

\textbf{The behavior of each robot} follows the same rules each task in order to determine the desired velocity $\bm{v}^d$ used in Problem \ref{prb:cbf}.
For each task $k=0,1,2,\dots,M$, the robot simply moves toward $\bm{p}_k$ at the maximum speed $v_{\max}$.
For $k=0$, the robot moves toward an ``idle location'' at the center of the domain.
This can be seen in Fig. \ref{fig:pm_frames}, which shows the position every robot (red circle) moving between the nodes (blue squares).
Note that to satisfy Assumption \ref{smp:hysteresis}, each robot remains on its assigned node $k$ until the node reaches a value of $R_k = 0$.

The values of $\frac{R_k(t)}{R_{\max}}$ are shown for each node in Fig. \ref{fig:pm_cost}.
These show a characteristic sawtooth shape, where each node linearly increases at a rate $A$ until a robot arrives and reduces the node at a rate of $A-B$.
The simulation snapshots in Fig. \ref{fig:pm_frames} show the emergence of dynamic teaming.
At $t=300$ seconds (left), two robots in the upper-right are approaching the single node with a large value of $R_k$, while a third robot moves toward the idle location, and the fourth robot is collecting information in the bottom-left.
There is only a single time around $t=600$ seconds where one node reaches its maximum information value of $R_k(t) = R_{\max}$.
This can be seen in the second image in Fig. \ref{fig:pm_frames}, where two robots are cooperatively approaching the blue node in the bottom left.
The final frame of Fig. \ref{fig:pm_frames} shows how the robots each split up to gather information from a single node later in the simulation.
This simulation demonstrates how we can efficiently allocate robots to tasks, even in the face of heterogeneity, under the strict constraints imposed by Assumption \ref{smp:hysteresis}.

\begin{figure}
    \centering
    \includegraphics[width=0.7\linewidth]{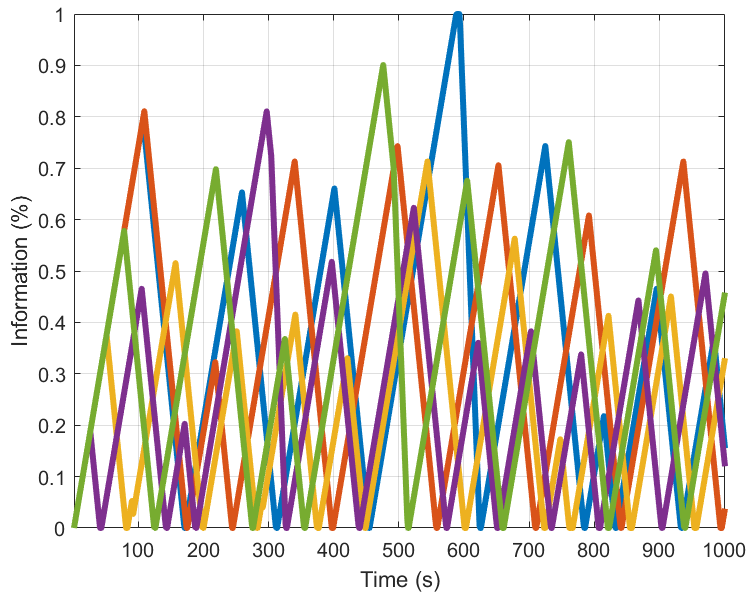}
    \caption{Normalized information accumulated at each node. A single node reaches the value of $R_{\max}$ at $t=600$ seconds.}
    \label{fig:pm_cost}
\end{figure}

\section{Conclusion} \label{sec:conclusion}

In this work, we presented an algorithm for the multi-robot task assignment using the framework of global games.
We derived the unique mixed Nash equilibrium for the case of all robots being homogeneous, and proved that the heterogeneous case reduces to a small number of homogeneous problems.
Furthermore, we provided an algorithm to determine the general solution for any number of robots and tasks in real time.
Finally, we presented two simulation results that show the applicability of our approach while also demonstrating its performance and resilience to the addition of tasks and removal of robots in real time.

Future work relies on more realistic dynamics for the robots and tasks, for example, experiments in hardware.
Introducing nonlinearity to the model through the signal, or learning a mapping from state to signal with a neural network, is another interesting direction of research.
Avoiding deadlock when some tasks that require the cooperation of multiple robots, e.g., cooperative transport of multiple objects, is another interesting direction of research--especially avoiding the case where one robot chatters back and forth between two task assignments.
Finally, finding probabilistic guarantees for the finite time completion of a task, such as never running out of energy, would be useful for long-duration deployment of physical systems in remote locations.

\bibliographystyle{unsrt}
\bibliography{mendeley,vsgc,ieeexplore}

\end{document}